\documentclass{article}

\PassOptionsToPackage{numbers, compress}{natbib}

\usepackage[preprint]{neurips_2026}

\usepackage[utf8]{inputenc} % allow utf-8 input
\usepackage[T1]{fontenc}    % use 8-bit T1 fonts
\usepackage[colorlinks,citecolor=blue]{hyperref}       % hyperlinks
\usepackage{url}            % simple URL typesetting
\usepackage{booktabs}       % professional-quality tables
\usepackage{amsfonts}       % blackboard math symbols
\usepackage{nicefrac}       % compact symbols for 1/2, etc.
\usepackage{microtype}      % microtypography
\usepackage[dvipsnames]{xcolor}         % colors
\usepackage{amsmath}
\usepackage{amsthm}
\usepackage[linesnumbered,ruled,lined,noend]{algorithm2e}
\usepackage{mathtools}
\usepackage{subcaption}
\usepackage{amssymb}
\usepackage{bm}
\usepackage{multirow}
\usepackage{tikz}
\usepackage{wrapfig}
\usetikzlibrary{shapes.geometric}

\DeclareMathOperator{\E}{\mathbb{E}}

\DeclareMathOperator{\R}{\mathbb{R}}

\DeclareMathOperator{\Var}{Var}
\DeclareMathOperator{\Cov}{Cov}
\DeclareMathOperator{\Corr}{Corr}
\DeclareMathOperator{\bias}{bias}

\newtheorem{theorem}{Theorem}

\newtheorem{proposition}{Proposition}

\title{Heuristic Pathologies and Further \\ Variance Reduction via Uncertainty Propagation \\ in the AIVAT Family of Techniques}

\author{%
  Juho Kim \\
  Computer Science Department \\
  Carnegie Mellon University \\
  \texttt{juhok@cs.cmu.edu} \\
  \And
  Tuomas Sandholm \\
  Computer Science Department, CMU \\
  Strategic Machine, Inc. \\
  Strategy Robot, Inc. \\
  Optimized Markets, Inc. \\
  \texttt{sandholm@cs.cmu.edu} \\
}

\begin{document}

\maketitle

\begin{abstract}
	How should an agent's performance in a multiagent environment be evaluated when there is a limited sample size or a high cost of running a trial?
	The AIVAT family of variance reduction techniques was proposed to address this challenge by introducing unbiased low-variance estimators of agents' expected payoffs.
	An important component of AIVAT is a heuristic value function that discriminates between potentially low- and high-value counterfactual histories.
	A notable gap in the literature is that there is little to no constraint or guideline on how the heuristic value function should be chosen or how uncertainty in its output should be handled.

	In our first contribution, we parameterize the heuristic value function to highlight AIVAT's potential vulnerabilities: a) the sample variance can be set pathologically low by directly applying gradient descent on the sample variance, and b) one can p-hack to draw a desired statistical conclusion via gradient descent/ascent on the test statistic.
	The main takeaway is that the heuristic value function should be fixed prior to observing the evaluation data!
	In our second contribution, we show how the heuristic uncertainty can be propagated to quantify the uncertainty of AIVAT estimates.
	It is then possible to further reduce the variance using inverse-variance weighted averaging, but AIVAT's unbiasedness guarantee may have to be sacrificed.
	In our experiments, we use a dataset of 10,000 poker hands to demonstrate our heuristic pathology and uncertainty results, with the latter yielding a 43.0\% reduction in the number of samples (poker hands) needed to draw statistical conclusions.
\end{abstract}

\section{Introduction}

Evaluating an agent's performance in a multiagent environment is often challenging, such as when running each trial is costly or time-consuming.
This is especially true when demonstrating the superhuman capability of an AI agent, requiring human experts to compete against the agent for a long time.
For example, the evaluation of the superhuman heads-up poker AI agent Libratus~\cite{brownandsandholm2018} lasted 20 days, morning to evening, and involved four human professionals competing in parallel for a \$200,000 prize pool.
The resource-intensive process of generating sufficient data to draw statistically significant conclusions cannot be avoided unless a low-variance estimator of the outcome is used.

The AIVAT~\cite{burchetal2018} family of variance reduction techniques was proposed to handle the often high-variance nature of extensive-form games.
AIVAT reduces the variance introduced by both nature and player actions in primarily two ways.
First, using a heuristic value function, AIVAT evaluates the potential values of counterfactual histories when counterfactual actions are applied to observed histories with known probabilities.
Second, AIVAT uses the fact that, regardless of a particular player's hidden information, others would have acted identically as they did in the original observations.
\citet{burchetal2018} showed that the AIVAT estimator is unbiased, and experimentally demonstrated a reduction in the required number of trials by ``more than a factor of 10'' to make the same statistical claims as when it is not used.
Furthermore, the power of AIVAT increases as the number of players whose strategy is known increases. However, it is also true that its power decreases as fewer player strategies are taken into account, with it being reduced to MIVAT~\cite{whiteandbowling2009} when only chance probabilities are known.

The `5 humans + 1 AI' experiment of the poker AI agent Pluribus~\cite{brownandsandholm2019}, evaluations of another poker AI agent DeepStack~\cite{moravciketal}, and several editions of the Annual Computer Poker Competitions (ACPC)~\cite{bardetal} represent the landmark applications of the AIVAT family of techniques.
The first application was particularly striking in that, although Pluribus finished the experiment with a negative payoff overall, AIVAT was able to show that Pluribus was, in fact, superhuman.

\subsection{Our contributions}

In this paper, we provide two types of contributions with regard to the AIVAT family of variance reduction techniques.
The first type is \textbf{cautionary}.
We expand on the fact that surprisingly little has been said about the constraints on how the heuristic value function can be developed.
In the proof of the unbiasedness of the advantage sum by~\citet{zinkevichetal2006}, which forms the basis of DIVAT~\cite{billingsandkan2006}, MIVAT~\cite{whiteandbowling2009}, and AIVAT~\cite{burchetal2018}, they state that ``any'' and ``all'' heuristic value functions yield an unbiased estimator of the true value.
\citet{whiteandbowling2009} suggest learning a linear value function from the sample data, where they optimize for the sample variance as a proxy of the true variance, but offer no additional guidelines on the learning process.
\citet{burchetal2018} use their AI agent's self-play values as the arbitrary fixed heuristic value function, which was also involved during the data-generating process of the same data on which the technique is applied.
In this paper, we highlight AIVAT's potential vulnerabilities by showing that it is possible to learn a heuristic value function that a) obtains pathologically low variance or b) p-hacks to falsely draw a desired statistical conclusion about an agent's performance.
Using the gameplay of Pluribus, we train such a function by parameterizing the heuristic outputs and applying gradient descent/ascent on the desired objective.
The main takeaway is that the heuristic value function \textit{should be fixed prior to observing the evaluation data}!
(The use of Pluribus data is purely for demonstration and should not be taken as a criticism of~\citet{brownandsandholm2019}; their results were correct.)

In our second contribution, we note that while AIVAT takes account of the uncertainty associated with player actions, in its usage of the heuristic value function, it introduces another source of uncertainty, namely, how certain the heuristic value function is in its outputs.
One may be more certain of some outputs of the heuristic value function while being less so about others.
This is certainly the case when the value function outputs are nondeterministically-approximated game-theoretic values (\textit{e.g.}, Monte Carlo rollouts and/or randomized clustering during abstraction) or are learned from existing data~\cite{whiteandbowling2009} and are predicting from inputs in the low- versus high-density region of the training distribution.
We quantify the uncertainty in terms of the variance and demonstrate how the heuristic value function uncertainty can be propagated to the estimate level, obtaining a measure of how uncertain a particular value estimate is.
We also show that inverse-variance weighting can be applied, where less weight is given to estimates with more uncertainty and vice versa, to achieve \textbf{further variance reduction}, albeit at the risk of incorporating some bias into the estimate.
Nonetheless, we a) show the necessary condition for this bias to be zero, b) demonstrate how this bias can be estimated, and c) contend that it is unlikely for game-playing agents to manipulate the bias to appear to play better.
Using the gameplay data of Pluribus, we report a reduction of up to 43.0\% in the number of required trials (\textit{i.e.}, poker hands) to reach statistical conclusions.
Our findings make multiagent evaluation more scalable.

\section{Notation and background}
\label{sec:related_works}

In this section, we define the notation used throughout the paper and provide the background on extensive-form games and the AIVAT family of variance reduction techniques.

\subsection{Extensive-form games}

In this paper, we focus our analysis on extensive-form games, but ideas from AIVAT can also be applied to other representations.
An extensive-form game has a finite set of players $P$ (including chance $p_c$) and histories $H$.
Every history is a sequence of actions played by each player $i \in P$, and is associated with a player $p(h)$ and a set of available actions $A(h)$.
$h \cdot a = h'$ denotes that applying $a \in A(h)$ at $h$ leads to $h'$.
If $p(h) = p_c$, then $f_c(h, a)$ gives a fixed probability distribution over each available action $a \in A(h)$.
Each terminal history $z \in Z \subseteq H$ has a utility $u_i(z)$ for every player $i$.

The imperfect information setting is represented by information sets $\mathcal{I}_i$: a partition of histories belonging to a non-chance player $i \in P \setminus \{p_c\}$.
A player $i$ cannot distinguish between $h, h' \in I \in \mathcal{I}_i$.
Therefore, $A(h) = A(h')$, and we denote the set of available actions at an information set by $A(I)$.
Each player $i$ plays with a strategy $\sigma_i(I)$ which assigns a probability distribution over $A(I)$.
Then a strategy profile $\sigma$ is defined as a tuple of all player strategies.
We use $\pi(h)$ to represent the probability of reaching $h$ given players play using $\sigma$.
The contribution of player $i$ to this probability is $\pi_i(h)$.

\subsubsection{Agent evaluation}

Agent evaluation in extensive-form games typically requires obtaining an estimate of the expected utility of a particular player $i$ on a given strategy profile $\sigma$:
\[
	\E_{z \in Z}[u_i(z)|\sigma] = \sum_{z \in Z} \pi(z) u_i(z).
\]
Before the advent of variance reduction techniques for extensive-form games, Monte Carlo samples $z_1, \hdots z_T$ were drawn independently to calculate the mean player utility,
\[
	\bar u_i = \frac{1}{T} \sum_{t=1}^T u_i(z_t),
\]
which is an unbiased estimator, \textit{i.e.},
\[
	\E[\bar u_i|\sigma] = \E_{z \in Z}[u_i(z)|\sigma],
\]
and has the following variance:
\[
	\Var [\bar u_i|\sigma] = \frac{1}{T} \Var [u_i(z)|\sigma].
\]

The multiagent environment's stochasticity and desired statistical significance level influence the choice of $T$.
One is therefore limited by the cost of running each trial, and it may be far too expensive to draw a statistically significant conclusion using the Monte Carlo method.

\subsection{Variance reduction techniques}

The AIVAT family of variance reduction techniques specializes the control variates method for agent evaluation in extensive-form games to give a low-variance estimate of any function $v(z)$, an example of which is player utility.
The current state-of-the-art variance reduction technique is AIVAT~\cite{burchetal2018}, which can be thought of as a combination of its two predecessors: imaginary observations~\cite{bowlingetal2008} and the advantage sum~\cite{zinkevichetal2006}.
The AIVAT family of techniques yields unbiased estimators of the value function.

\subsubsection{Control variates}

The control variates method~\cite[Ch.~4]{glasserman} is a standard way to reduce variance in Monte Carlo methods.
We begin by giving a brief description in the language of agent evaluation.
In this setting, we seek to estimate $\E_{z \in Z}[v(z)|\sigma]$.
Suppose the existence of another value function $w(\cdot)$ where $\omega = \E_{z \in Z}[w(z)|\sigma]$ is known.
Then, the following is an unbiased estimator of $\E_{z \in Z}[v(z)|\sigma]$:
\[
	\hat v(z) = v(z) - c(w(z) - \omega)
\]
for any choice of constant $c$.
Its variance is as follows:
\begin{equation}
	\Var(\hat v(z)) = Var(v(z)) + c^2 Var(w(z)) - 2c \Cov(v(z), w(z)).
	\label{eqn:control-variates}
\end{equation}
When $c^2 Var(w(z)) - 2c \Cov(v(z), w(z)) \le 0$, variance reduction is achieved.
The optimal choice of $c$ can be derived by differentiating (\ref{eqn:control-variates}) with respect to $c$:
\[
	c^* = \frac{\Cov(v(z), w(z))}{\Var(w(z))},
\]
which results in
\[
	\Var(\hat v(z)) = (1 - \Corr(v(z), w(z))^2) \Var(v(z)).
\]

\subsubsection{Advantage sum}

The advantage sum technique uses control variates with a heuristic value function to reduce the variance of the estimate, and is of the following form:
\[
	\hat{v}(z) = v(z) - \hat{v}_c(z),
\]
where $\hat{v}(\cdot)$, $v(\cdot)$, and $\hat{v}_c(\cdot)$ denote the estimate, value function, and correction term, respectively.
The correction term, which can be thought of as control variates, is defined as follows:
\[
	\hat{v}_c(z) = \sum_{h \cdot a \in K(z)}{\left(v'(h \cdot a) - \sum_{a' \in A(h)} f_c(h, a') v'(h \cdot a')\right)},
\]
with $v'(\cdot)$ the heuristic value function and $K(z)$ the set of histories preceding $z$ where the probability distribution of available actions at the immediate parent history is known.
MIVAT~\cite{whiteandbowling2009} represents a special case of the advantage sum where no player's action probabilities, except those of nature, are known.
In this particular context, the subtraction of the correction term can be thought of as canceling the effect of luck (or lack thereof).
No matter the choice of the heuristic value function $v'(\cdot)$,~\citet{zinkevichetal2006} showed that the advantage sum (and hence MIVAT) is an unbiased estimator of the expected value, \textit{i.e.}, $\E_{z \in Z}[\hat{v}(z)|\sigma] = \E_{z \in Z}[v(z)|\sigma]$.

\citet{whiteandbowling2009} proposed using a linear function as the heuristic value function, which is to be trained on existing sample data.
In doing so, they proposed to minimize the sample variance as a proxy of the true variance, hence the following optimization problem for the mean-squared error:
\begin{equation}
	\underset{\hat v: Z \mapsto \R}{\textbf{Minimize: }} \sum_{t=1}^T\left(\hat v(z_t) - \frac{1}{T} \sum_{t'=1}^T \hat v(z_{t'})\right)^2.
	\label{eqn:optimization-problem}
\end{equation}
They derived a closed-form formula of an optimal linear value function given a feature engineering function and applied it to heads-up fixed-limit, heads-up no-limit, and 6-max fixed-limit poker, observing reductions of up to about 62\%, 23\%, and 18\%, respectively, in the standard deviation.

\subsubsection{Imaginary observations}

Imaginary observations do not use a heuristic value function or control variate terms.
Instead, given a trial outcome $z$, it directly applies the value function on a subset of terminal nodes and obtains an estimate of the expected payoff using importance sampling.
It is well-suited for estimating the expected payoff of alternative strategies (which they dubbed the off-policy case).
The imaginary observation estimator is unbiased in both the on-policy and off-policy cases with full information; however, this guarantee is lost with partial information~\cite{bowlingetal2008}.
Nevertheless, imaginary observations in the partial information case remain a useful tool as they can yield a low-variance estimate.

\subsubsection{AIVAT}

AIVAT, visualized in Appendix~\ref{sec:aivat}, is a generalization of the advantage sum where imaginary observations are applied to each considered history.
It is of the following form:
\[
	\hat{v}(z) = \hat{v}_b(z) + \hat{v}_c(z),
\]
where $\hat{v}_b(z)$ is the base term:
\[
	\hat{v}_b(z) = \frac{\sum_{z' \in U(z)} \pi(z') v(z')}{\sum_{z' \in U(z)} \pi(z')},
\]
and $\hat{v}_c(z)$ is the correction term:
\[
	\hat{v}_c(z) = \sum_{h \cdot a \in K(z)} \left(\frac{\sum_{a' \in A(U(h))} \sum_{h' \in U(h)} \pi(h' \cdot a') v'(h' \cdot a')}{\sum_{h' \in U(h)} \pi(h')} - \frac{\sum_{h' \in U(h)} \pi(h' \cdot a) v'(h' \cdot a)}{\sum_{h' \in U(h)} \pi(h' \cdot a)}\right).
\]

\noindent $U(h)$ is defined as a set of histories differing from $h$ only by the private information belonging to $p(h)$.
$A(U(h))$ is the common set of actions available to that player.
\citet{burchetal2018} showed that AIVAT is an unbiased estimator regardless of $v'(h)$ and proposed using an AI agent's self-play values as the outputs of the heuristic value function, which can also be used to generate the very data being evaluated.

\section{Heuristic pathologies}
\label{sec:heuristic-pathology}

We are ready to present our new results.
We begin by presenting a \textbf{cautionary tale}, highlighting AIVAT's potential vulnerabilities.
We show that one can learn a heuristic value function to a) obtain pathologically low variance or b) p-hack to falsely draw a desired statistical conclusion about an agent's performance.
Our results underscore the need to fix the heuristic value function before evaluation.

For convenience, we simplify the expression for the AIVAT estimate $\hat v(\cdot)$ by rewriting it as an affine function of the outputs of the value function $v'(\cdot)$, as follows:
\begin{equation}
	\hat v(z) = b(z) + \sum_{h \in H} \mathbf{c}(z)_h v'(h),
	\label{eqn:simplified}
\end{equation}
where $b(z)$ is the affine shift and $\mathbf{c}(z) \in \R^H$ is the vector of coefficients of each heuristic output $v'(\cdot)$.
Depending on the game, the size of $H$ may be huge; however, for our purpose, we only need to consider the histories and counterfactual histories $h$ encountered during agent evaluation (\textit{i.e.}, where $\mathbf{c}(z)_h \ne 0$), which is usually a small fraction of the original game tree size if the game is large.
The detailed derivation process, as well as how $b(z)$ and $\mathbf{c}(z)$ are defined, are relegated to Appendix~\ref{sec:simplification}.

In our development of a pathological heuristic value function, we define a vector $\bm{\theta} \in \R^H$ to parameterize each heuristic output as follows:
\begin{equation}
	v_{\bm{\theta}}'(h) = \theta_h.
	\label{eqn:parameterized-value-function}
\end{equation}
By parameterizing as such, we give the heuristic value function the maximum expressive power.
Again, the size of $H$ can be huge.
However, we only need to parameterize the (counterfactual) histories encountered during agent evaluation.
We can further simplify the AIVAT estimate expression:
\begin{equation}
	\hat v_{\bm{\theta}}(z) = b(z) + \sum_{h \in H} \mathbf{c}(z)_h v_{\bm{\theta}}'(h) = b(z) + \langle \mathbf{c}(z), \bm{\theta} \rangle.
	\label{eqn:estimate}
\end{equation}

\subsubsection{Optimizing for the sample variance}

Using (\ref{eqn:estimate}), the proxy objective of optimization problem given by~\citet{whiteandbowling2009} in (\ref{eqn:optimization-problem}) can be refined as follows:
\begin{align*}
	& \underset{\bm{\theta} \in \R^H}{\textbf{Minimize: }} C(\bm{\theta})
		= \sum_{t=1}^T\left(\hat v_{\bm{\theta}}(z_t) - \frac{1}{T} \sum_{t'=1}^T \hat v_{\bm{\theta}}(z_{t'})\right)^2 = \sum_{t=1}^T\left(\left(b(z_t) - \bar{b}\right) + \left< \mathbf{c}(z_t) - \bar{\mathbf{c}}, \bm{\theta} \right> \right)^2,
\end{align*}
where $\bar{b} = \frac{1}{T} \sum_{t=1}^T b(z_t)$ and $\bar{\mathbf{c}} = \frac{1}{T} \sum_{t=1}^T \mathbf{c}(z_{t})$. (A more detailed derivation is given in Appendix~\ref{sec:detailed}.)

\begin{proposition}
	On a given set of trials, there exists an optimal parameter vector $\bm{\theta}^*$ that achieves the lowest possible sample variance of the estimates.
	\label{pro:optimal}
\end{proposition}
The proof is in Appendix~\ref{sec:optimal}.
It follows from the fact that finding the pathological heuristic outputs that minimize the sample variance of a given dataset is a least-squares problem.

\subsubsection{Optimizing for the \textit{t}-statistic}

However, optimizing for the sample variance may not be very interesting, as we do not explicitly control what we can say about the given data.
More specifically, one of the main purposes of using variance reduction techniques is to draw statistical conclusions about an agent's performance.
This is usually done~\cite{brownandsandholm2019} via a one-sided \textit{t}-test from which a p-value is computed and compared against a desired statistical confidence level.
A p-value can potentially be hacked by either minimizing or maximizing the \textit{t}-statistic, as shown in the optimization problem below.
\begin{align*}
	\underset{\bm{\theta} \in \R^H}{\textbf{Optimize: }} \frac{\bar{v_{\bm{\theta}}} - \mu_0}{s_{\bm{\theta}}/\sqrt{T}},
\end{align*}
where $\bar{v_{\bm{\theta}}} = \sum_{t=1}^T \hat v_{\bm{\theta}}(z_t)$ and $s_{\bm{\theta}}^2 = \frac{\sum_{t=1}^T (\hat v_{\bm{\theta}}(z_t) - \bar{v_{\bm{\theta}}})^2}{T - 1}$.
We apply gradient descent/ascent on the parameters to minimize or maximize our objective.

\subsection{Experiments on heuristic pathologies}
\label{sec:experiment}

We conducted experiments on the publicly released poker hand history data from the Pluribus experiment~\cite{brownandsandholm2019}.
The rules of Texas hold'em are given in Appendix~\ref{sec:texas-holdem}.

\subsubsection{Optimizing for the sample variance}

\begin{wraptable}{R}{0.33\linewidth}
	\centering
	\caption{
		Results on the Pluribus data in milli-big blinds per hand (mbb/h) using a pathological heuristic value function that minimizes the sample variance.
		$SE$ stands for the standard error of the mean.
	}
	\label{tab:experiment-1}
	\begin{tabular}{c|cc}
		\toprule
		Player & Win rate & $SE$ \\
		\midrule
		Budd & 2059 & 50 \\
		MrWhite & 2043 & 37 \\
		MrOrange & 2053 & 29 \\
		Hattori & 2045 & 68 \\
		MrBlue & 2062 & 26 \\
		Pluribus & 2062 & 25 \\
		MrPink & 2048 & 32 \\
		Joe & 2064 & 64 \\
		Bill & 2061 & 30 \\
		MrBlonde & 2040 & 49 \\
		Eddie & 2049 & 34 \\
		MrBrown & 2035 & 67 \\
		Gogo & 2006 & 113 \\
		ORen & 2034 & 90 \\
		\bottomrule
	\end{tabular}
\end{wraptable}

In our first experiment, we aim to produce a result that is of extremely low variance that does not pertain to reality.
To do so, we optimize on the sample variance as the proxy objective.
Since Pluribus's action distribution is not known, our AIVAT implementation reduces to MIVAT.
We trained a heuristic value function, outputting an estimate of the expected utilities at every player position.
Although the least-squares solution could theoretically have been solved, the transformed data matrix $[(\mathbf{c}(z_t) - \bar{\mathbf{c}})^\top]_{t \in \{1, \dots, T\}}$ was too large to fit into memory, so we applied gradient descent on the parameters for 250 iterations with the Adam optimizer ($\eta = 100, \beta_1 = 0.9, \beta_2 = 0.999, \lambda = 0$).
The unusually high learning rate ($\eta$) was chosen as lower (and more conventional) choices (\textit{e.g.}, $0.001$) converged too slowly.

Using the learned parameters, we were able to achieve \textbf{pathologically} low variance estimates of all players' expected utilities, as shown in Table~\ref{tab:experiment-1}.
The results show that every player won by an extremely high margin of over 2,000 mbb/h.
Considering that the poker community typically characterizes win rates of 100 mbb/h as immense, these values are clearly unrealistic.
They also violate the zero-sum constraint; this is because the heuristic value function was trained without such a constraint.
This experiment shows that, when the data is fixed and the heuristic value function is directly optimized using the proxy objective given by~\citet{whiteandbowling2009}, one can produce a nonsensical result that does not pertain to reality, especially with enough degrees of freedom to be exploited.

\subsubsection{Optimizing for the \textit{t}-statistic}

In our second experiment, we explore how to falsely make desired claims about the given data by p-hacking via gradient descent/ascent to either minimize or maximize the \textit{t}-statistics.
Here, we again train a parameterized heuristic value function except that, this time, a player's \textit{t}-statistic is being optimized.
For each run, when the \textit{t}-statistic is being minimized, the null hypothesis is that a player did not have a negative win rate.
Conversely, when a player's \textit{t}-statistic is being maximized, the null hypothesis is that the player did not have a positive win rate.
The goal of this experiment is to obtain a small enough p-value for each hypothesis test to draw statistically significant conclusions that every player both won and lost.
The parameters were optimized for 10 iterations with the Adam optimizer ($\eta = 100, \beta_1 = 0.9, \beta_2 = 0.999, \lambda = 0$).
Again, the common choices for $\eta$ converged too slowly.

The results are shown in Table~\ref{tab:experiment-2}.
For every player, it was possible to train separate heuristic value functions, one for winning and another for losing, using which one could show that they both won and lost on the same data with overwhelmingly low p-values.
The implication of this is that although the AIVAT family of variance reduction techniques is guaranteed to be unbiased, one can still craft a heuristic value function that supports whatever conclusion one wants to draw when the data being evaluated is known a priori.
One interpretation of this result is that, when an \textbf{adversary} controls the heuristic value function and is aware of the data being evaluated, the adversary can make the evaluator draw conclusions that are nonsensical, contradictory, or incorrect.

\begin{table}[t!]
	\centering
	\caption{Results on the Pluribus data using pathological heuristic value functions for each \textit{t}-statistic.}
	\label{tab:experiment-2}
	\begin{tabular}{c|cc|cc}
		\toprule
		\multirow{2}{*}{Player} & \multicolumn{2}{c|}{Losing} & \multicolumn{2}{c}{Winning} \\
		& \textit{t}-statistic & p-value & \textit{t}-statistic & p-value \\
		\midrule
		Budd & -49.538 & \textless10\textsuperscript{-373} & 55.012 & \textless10\textsuperscript{-433} \\
		MrWhite & -76.612 & \textless10\textsuperscript{-816} & 75.638 & \textless10\textsuperscript{-802} \\
		MrOrange & -95.291 & \textless10\textsuperscript{-1294} & 89.609 & \textless10\textsuperscript{-1188} \\
		Hattori & -37.584 & \textless10\textsuperscript{-212} & 39.327 & \textless10\textsuperscript{-226} \\
		MrBlue & -87.433 & \textless10\textsuperscript{-1207} & 93.654 & \textless10\textsuperscript{-1336} \\
		Pluribus & -95.050 & \textless10\textsuperscript{-1399} & 92.900 & \textless10\textsuperscript{-1353} \\
		MrPink & -69.132 & \textless10\textsuperscript{-767} & 68.549 & \textless10\textsuperscript{-757} \\
		Joe & -44.570 & \textless10\textsuperscript{-278} & 42.285 & \textless10\textsuperscript{-259} \\
		Bill & -70.778 & \textless10\textsuperscript{-814} & 70.991 & \textless10\textsuperscript{-818} \\
		MrBlonde & -70.923 & \textless10\textsuperscript{-606} & 68.642 & \textless10\textsuperscript{-582} \\
		Eddie & -66.739 & \textless10\textsuperscript{-710} & 69.562 & \textless10\textsuperscript{-756} \\
		MrBrown & -42.189 & \textless10\textsuperscript{-249} & 44.325 & \textless10\textsuperscript{-266} \\
		Gogo & -22.699 & \textless10\textsuperscript{-77} & 19.617 & \textless10\textsuperscript{-63} \\
		ORen & -20.375 & \textless10\textsuperscript{-73} & 20.807 & \textless10\textsuperscript{-76} \\
		\bottomrule
	\end{tabular}
\end{table}

Note that we are not criticizing the application of AIVAT by~\citet{brownandsandholm2019}.
We are simply using their dataset to demonstrate how invalid conclusions can be drawn due to heuristic pathologies.
These pathologies also shed light on the need to \textbf{fix} the heuristic value function prior to evaluation.
While we are unaware of existing AIVAT applications that violate this, to our knowledge, we are the first to point out vulnerabilities of this nature in the AIVAT family of variance reduction techniques.

In the next section, we provide a practical methodology of training and evaluating the AIVAT family of estimators, especially in the face of data scarcity, which leads to additional variance reduction.

\section{Further variance reduction using heuristic uncertainty}
\label{sec:heuristic-uncertainty}
We now move to our second contribution, which is a new way of obtaining additional variance reduction on top of AIVAT.
We continue our discussion about the heuristic value function in variance reduction techniques by approaching it from a different angle: uncertainty of heuristic outputs.
We use variance to quantify uncertainty.
We continue from the simplified expression shown in (\ref{eqn:simplified}).
Defining the covariance matrix $\mathbf{\Sigma}(z)$ where, for each $(h_1, h_2) \in H \times H$, $\mathbf{\Sigma}(z)_{h_1, h_2} = \Cov(v'(h_1), v'(h_2))$,
\begin{equation}
	\Var(\hat v(z)) = \Var\left(b(z) + \sum_{h \in H} \mathbf{c}(z)_h v'(h)\right) = \Var\left(\sum_{h \in H} \mathbf{c}(z)_h v'(h)\right) = \mathbf{c}(z)^\top \mathbf{\Sigma}(z) \mathbf{c}(z). \label{eqn:estimate-variance}
\end{equation}
Note that, in practice, $\mathbf{c}(z)$ is sparse (see Appendix~\ref{sec:simplification}), and hence $\mathbf{\Sigma}(z)$ can be implemented as sparse, setting irrelevant rows and columns to zeros.
One can simplify the expression further, as in the following, by assuming that the heuristic value outputs are uncorrelated:
\begin{equation}
	\Var(\hat v(z)) = \sum_{h \in H} \mathbf{c}(z)_{h}^2 \Var(v'(h)).
	\label{eqn:estimate-variance-uncorrelated}
\end{equation}
Note that the above assumption was \textbf{not} made in our main experimental results, presented later.

Just as we are given $T$ independent trials $z_1, \hdots, z_T$ and assume the values $v(z_1), \hdots, v(z_T)$ are independent and identically distributed (i.i.d.), we make the standard~\cite{brownandsandholm2019} assumption that the AIVAT estimates are also i.i.d.
Then, the variance of the arithmetic mean of AIVAT estimates $\bar v = \frac{1}{T} \sum_{t=1}^T \hat v(z_t)$ is as follows:
\begin{equation}
	\Var(\bar v) = \frac{\sum_{t=1}^T \Var(\hat v(z_t))}{T^2}.
	\label{eqn:arithmetic-mean-variance}
\end{equation}
We can treat the variance of the mean as a proxy objective to be minimized.
It is possible to improve from (\ref{eqn:arithmetic-mean-variance}) by assigning an (unnormalized) weight $w_t$ to each estimate to obtain the weighted average, \textit{i.e.},
\[
	\bar v^* = \frac{\sum_{t=1}^T w_t \hat v(z_t)}{\sum_{t=1}^T w_t}.
\]
Consider inverse-variance weighting (IVW)~\cite[Ch.~4]{hartungetal2011}, which puts less weight on those with higher uncertainty and vice versa: $w_t = \frac{1}{\Var(\hat v(z_t))}$.
The variance of the IVW average is then
\begin{equation}
	\Var(\bar v^*) = \frac{1}{\sum_{t=1}^T \frac{1}{\Var(\hat v(z_t))}}.
	\label{eqn:inverse-variance-weighted-mean-variance}
\end{equation}

\begin{proposition}
	Assuming that the AIVAT estimates are independent, their IVW average yields the minimum variance amongst all weighted averages.
	\label{pro:ivw}
\end{proposition}
The proof is in Appendix~\ref{sec:ivw}.
It follows from a well-known result in statistics; cf.~\citet[Ch.~4]{hartungetal2011}.
Thus, IVW improves upon uniform weighting.
When the estimates being averaged over are all equally uncertain, (\ref{eqn:arithmetic-mean-variance}) and (\ref{eqn:inverse-variance-weighted-mean-variance}) are equivalent.
Ideally, we would like the IVW average to also be an unbiased estimator.
This is true when the weights are independent of the data being averaged over, but not in general, as shown in the following proposition:
\begin{proposition}
	When the weights are independent of the data being averaged over, the weighted average of AIVAT estimates is unbiased, \textit{i.e.}, $\E[\bar v^*|\sigma] = \E_{z \in Z}[v(z)|\sigma]$.
	But when the weights are correlated with the data being averaged over, the weighted average of AIVAT estimates is, in general, biased, \textit{i.e.}, $\E[\bar v^*|\sigma] \ne \E_{z \in Z}[v(z)|\sigma]$. Assuming that the data points are i.i.d., the asymptotic bias is $\Cov(w, \hat v(z)|\sigma)/\E[w]$, where $w$ and $\hat v(z)$ are random variables representing the unnormalized weights and estimated expected utilities.
	\label{pro:estimator-bias}
\end{proposition}
The proof is in Appendix~\ref{sec:estimator-bias}.
To address this, we can \textbf{estimate} this bias, and our estimator can still remain useful by showing that the estimated bias is much smaller than the reduction in the standard error of the weighted mean.
While we can technically use this estimate to offset the IVW mean, doing so adds a highly non-trivial term to the variance.

It is important to note that the estimator we introduce does \textbf{not} inherently introduce bias.
Indeed, whether or not our estimator is biased depends on the specific learning algorithm used to train the heuristic value function, since this determines the AIVAT estimates we obtain and their corresponding weights.
Additionally, fixing the specific learning algorithms can help us derive even stronger theoretical bounds.
Later in our experiments, the model we utilized as the heuristic value function is \textit{Gaussian process regressor (GPR)}~\cite{williamsandrasmussen}, which relies on Gaussian assumptions, of which one of the fundamental statistical properties is that the mean and the variance are structurally independent.
The estimated uncertainty depends only on the input space, so the output and estimated uncertainty have zero covariance under Gaussian priors, and thus our estimator is unbiased.
The same holds for other models we use in Appendix~\ref{sec:observations}: \textit{Bayesian ridge (BR)}~\cite{tipping} and \textit{automatic relevance determination (ARD)}~\cite{mackay}, which operate under the same foundational statistical assumptions.

Besides, it is unclear how a player could even play to manipulate such a weighting scheme, since the bias depends on the correlation between the IVW weights and the expected utility estimates.
It would require them to play differently while predicting the confidence of the heuristic value function chosen by the evaluator in its outputs for different histories.
A similar bias-variance tradeoff was previously explored in the context of imaginary observations (a predecessor of AIVAT) in the case of partial information.
We also demonstrate this tradeoff in our experiments below.

\subsection{Experiments on heuristic uncertainty propagation}
\label{sec:experiment2}

We conducted experiments on the publicly available poker hand history data from the Pluribus experiment~\cite{brownandsandholm2019}.
Since Pluribus's action distribution is not known to us, our implementation of AIVAT reduces to MIVAT.
In learning the heuristic value function for our MIVAT estimator, we chose the \textit{Gaussian Process Regressor (GPR)}~\cite{williamsandrasmussen} with Dot Product and White kernels.
We refer to this MIVAT estimator as MIVAT-GPR.
For feature engineering, we took an approach similar to that of~\citet{whiteandbowling2009}, detailed in Appendix~\ref{sec:feature-engineering}.
The Pluribus data consists of 10,000 poker hands.
We applied k-fold cross-validation with $k=10$ to calculate the estimates for the entirety of the data.
In each fold, a heuristic value function was trained using all relevant history-payoff pairs in the training set, and then the resulting estimator was evaluated on the test set.
While training MIVAT-GPR, in each fold, 1,000 hands were subsampled from the training set due to computational constraints.

\begin{table}
	\centering
	\caption{
		Results on the Pluribus data in milli-big blinds per hand (mbb/h).
		$SE$ stands for the standard error of the (weighted) mean, and `Est.' is an abbreviation for `Estimated'.
	}
	\label{tab:results-pluribus}
	\begin{tabular}{cc|ccc}
		\toprule
		Estimator & Weighting & Win rate & $SE$ & Est. bias \\
		\midrule
		\multirow{2}{*}{MIVAT-GPR (ours)} & Uniform & -25 & 99 & -- \\
		& IVW & -22 & \textbf{75} & 3 \\
		\midrule
		Raw & Uniform & -70 & 88 & -- \\
		MIVAT-WB~\cite{whiteandbowling2009} & Uniform & -98 & 85 & --\\
		AIVAT~\cite{burchetal2018,brownandsandholm2019} & Uniform & 48 & 25 & -- \\
		\bottomrule
	\end{tabular}
\end{table}

Our objective for the experiment is not to draw statistical conclusions about Pluribus's superhuman performance, which, realistically, would require access to Pluribus's strategies.
Instead, we seek to demonstrate that IVW gives rise to an estimate that has a lower standard error of the (weighted) mean than uniform weighting (more details are available in Appendix~\ref{sec:standard-error-of-the-weighted-mean}).
In our results, we also include the results for AIVAT, as reported by~\citet{brownandsandholm2019}.
It would be unfair to compare our performance with theirs, as they had access to Pluribus's action probabilities, which vastly increase AIVAT's variance-reduction power compared to when only the chance probabilities are known (which is what we know).
Additionally, the results using MIVAT with the linear heuristic value function obtained through the steps introduced by~\citet{whiteandbowling2009} were included.
We refer to this estimator as MIVAT-WB.
In Appendix~\ref{sec:closed-form-solution}, we derive this function for the more general case of AIVAT, but we also critique its assumptions and lack of regularization.

Table~\ref{tab:results-pluribus} shows the performance of our MIVAT-GPR estimator, along with that of the baselines.
When taking the simple mean, MIVAT-GPR does worse than MIVAT-WB, but this is understandable, as we used far less data to train GPR due to computational constraints.
Also, under uniform averaging, MIVAT-GPR performs worse than even when no estimator is applied (see `Raw').
It is clear that, due to the much smaller training data used to train its heuristic value function, GPR is quite inadequate in calculating the heuristic value of some counterfactual situations.
However, under IVW averaging, our MIVAT-GPR estimator is also observed to far outperform itself under uniform averaging and MIVAT-WB.
Indeed, MIVAT-GPR notably achieved approximately 24.5\% reduction in the standard error of the (weighted) mean compared to when uniform averaging was used, corresponding to roughly 43.0\% reduction in the number of hands required to reach the same statistical significance.
Additionally, the estimated bias of the IVW estimate is about an order of magnitude smaller than the reduction in the standard error of the weighted mean.
This is evidence that IVW, which puts lower weights on outputs with higher uncertainty, can help achieve a low-variance estimate of expected player utilities in games even when heuristics are sometimes poor.

Another result to note is that MIVAT-WB~\cite{whiteandbowling2009} performs noticeably worse compared to MIVAT-GPR under IVW and not much better than when no estimator is used.
At first glance, this is surprising, as their machine learning procedure is supposed to produce a linear heuristic value function that is optimal in the sense that it minimizes the sample variance of the training set.
In Appendix~\ref{sec:closed-form-solution}, we note its lack of regularization, and it is clear that this learning process is too aggressive and can be prone to overfitting, which hurts the estimator's performance.
To our knowledge, we are the first to note that much simpler procedures of learning the heuristic value function can outperform the linear value function of~\citet{whiteandbowling2009}.
Additional experimental results are detailed in Appendix~\ref{sec:observations}.

\section{Conclusions and future research}

In this paper, we studied the roles played by the underspecified heuristic value function in the AIVAT family of variance reduction techniques.
First, we developed a heuristic value function that allowed us to pathologically lower the variance and draw desired statistical conclusions about agents' performances.
Our results showed that, when the data is fixed and known, an adversary can craft a heuristic value function that can lead to drawing nonsensical, misleading, or contradictory statistical conclusions.
Thus, the heuristic value function should be fixed \textit{prior to evaluation}!

Second, we showed that the uncertainty of the heuristic value function outputs, if known, can be propagated to find the uncertainty of AIVAT estimates.
Our contribution on heuristic uncertainty is particularly notable in that it leverages extra-game-theoretical considerations to yield further variance reduction.
We demonstrated that IVW further reduces the variance and yields a 43.0\% reduction in the number of samples (poker hands) required to draw statistical conclusions.
While this may come at a cost of possible bias, we can estimate this bias, and fixing the learning algorithm for the heuristic value function can yield stronger theoretical guarantees.

A possible future direction is developing new methods to quantify the uncertainty of self-play values approximated using non-deterministic game-solving algorithms, which can then be used by AIVAT with IVW.

\section*{Acknowledgements}

This work has been supported by the Vannevar Bush Faculty Fellowship ONR N00014-23-1-2876, National Science Foundation grant RI-2312342, and NIH award A240108S001.
Any opinions, findings, and conclusions or recommendations expressed in this material are those of the authors and do not necessarily reflect the views of the funding agencies.

\bibliographystyle{abbrvnat}
\bibliography{neurips_2026}

%%%%%%%%%%%%%%%%%%%%%%%%%%%%%%%%%%%%%%%%%%%%%%%%%%%%%%%%%%%%

\newpage
\appendix

\section{Visualization of AIVAT}
\label{sec:aivat}

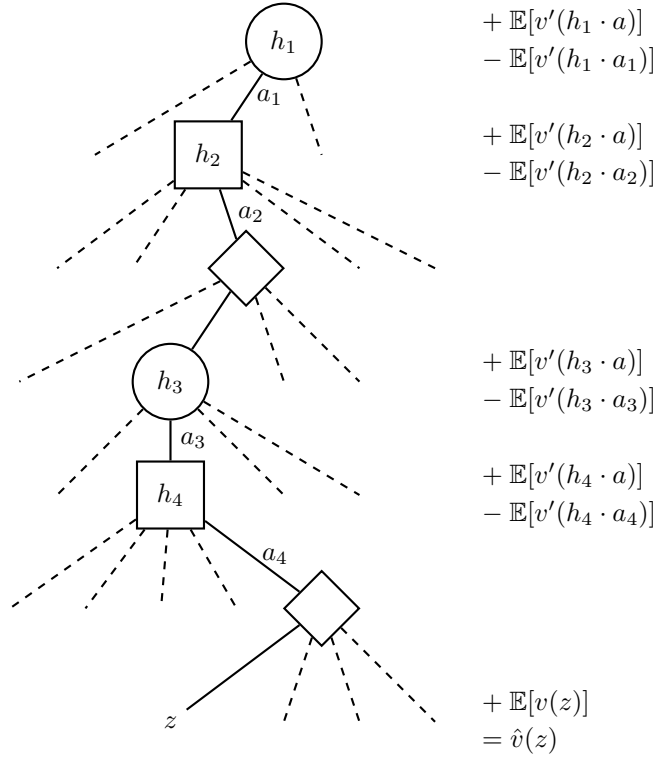
\begin{figure}[th!]
	\centering
	\begin{tikzpicture}
		\node[circle, draw, minimum size=1cm, thick] (h1) at (0,0) {$h_1$};
		\node[rectangle, draw, minimum size=0.886cm, thick] (h2) at (-1,-1.5) {$h_2$};
		\node[diamond, draw, minimum size=1cm, thick] (c1) at (-0.5,-3) {};
		\node[circle, draw, minimum size=1cm, thick] (h3) at (-1.5,-4.5) {$h_3$};
		\node[rectangle, draw, minimum size=0.886cm, thick] (h4) at (-1.5,-6) {$h_4$};
		\node[diamond, draw, minimum size=1cm, thick] (c2) at (0.5,-7.5) {};
		\node (z) at (-1.5,-9) {$z$};

		\draw[thick] (h1) -- (h2) node[midway, right] {$a_1$};
		\draw[thick] (h2) -- (c1) node[midway, right] {$a_2$};
		\draw[thick] (c1) -- (h3);
		\draw[thick] (h3) -- (h4) node[midway, right] {$a_3$};
		\draw[thick] (h4) -- (c2) node[midway, right] {$a_4$};
		\draw[thick] (c2) -- (z);

		\draw[dashed, thick] (h1) -- +(-2.5,-1.5);
		\draw[dashed, thick] (h1) -- +(0.5,-1.5);

		\draw[dashed, thick] (h2) -- +(-2,-1.5);
		\draw[dashed, thick] (h2) -- +(-1,-1.5);
		\draw[dashed, thick] (h2) -- +(2,-1.5);
		\draw[dashed, thick] (h2) -- +(3,-1.5);

		\draw[dashed, thick] (c1) -- +(-3,-1.5);
		\draw[dashed, thick] (c1) -- +(0.5,-1.5);
		\draw[dashed, thick] (c1) -- +(1.5,-1.5);

		\draw[dashed, thick] (h3) -- +(-1.5,-1.5);
		\draw[dashed, thick] (h3) -- +(1.5,-1.5);
		\draw[dashed, thick] (h3) -- +(2.5,-1.5);

		\draw[dashed, thick] (h4) -- +(-2.125,-1.5);
		\draw[dashed, thick] (h4) -- +(-1.125,-1.5);
		\draw[dashed, thick] (h4) -- +(-0.125,-1.5);
		\draw[dashed, thick] (h4) -- +(0.875,-1.5);

		\draw[dashed, thick] (c2) -- +(-0.5,-1.5);
		\draw[dashed, thick] (c2) -- +(0.5,-1.5);
		\draw[dashed, thick] (c2) -- +(1.5,-1.5);

		\node[anchor=west, xshift=1cm] at (1.5, 0.25) {$+\ \mathbb{E}[v'(h_1 \cdot a)]$};
		\node[anchor=west, xshift=1cm] at (1.5, -0.25) {$-\ \mathbb{E}[v'(h_1 \cdot a_1)]$};
		\node[anchor=west, xshift=1cm] at (1.5, -1.25) {$+\ \mathbb{E}[v'(h_2 \cdot a)]$};
		\node[anchor=west, xshift=1cm] at (1.5, -1.75) {$-\ \mathbb{E}[v'(h_2 \cdot a_2)]$};
		\node[anchor=west, xshift=1cm] at (1.5, -4.25) {$+\ \mathbb{E}[v'(h_3 \cdot a)]$};
		\node[anchor=west, xshift=1cm] at (1.5, -4.75) {$-\ \mathbb{E}[v'(h_3 \cdot a_3)]$};
		\node[anchor=west, xshift=1cm] at (1.5, -5.75) {$+\ \mathbb{E}[v'(h_4 \cdot a)]$};
		\node[anchor=west, xshift=1cm] at (1.5, -6.25) {$-\ \mathbb{E}[v'(h_4 \cdot a_4)]$};
		\node[anchor=west, xshift=1cm] at (1.5, -8.75) {$+\ \mathbb{E}[v(z)]$};
		\node[anchor=west, xshift=1cm] at (1.5, -9.25) {$= \hat{v}(z)$};
	\end{tikzpicture}
	\caption{A diagram representing AIVAT. Circles, squares, and diamonds denote chance nodes, player-1 nodes, and player-2 nodes, respectively. This diagram illustrates a situation where only the action probabilities of the nodes belonging to nature and Player 1 are known. The right side of the diagram shows the terms involved in the calculation of the AIVAT estimate. Here, $U(h) = \{h\}$ for the sake of simplicity.}
	\label{fig:aivat}
\end{figure}

A simple visualization of AIVAT is given in Figure~\ref{fig:aivat}.

\section{Simplifying the expression for AIVAT}
\label{sec:simplification}

We simplified the expression for the AIVAT estimate $\hat v(\cdot)$ by rewriting it as an affine function of the outputs of the value function $v'(\cdot)$ in Section~\ref{sec:heuristic-pathology}.
The detailed derivation process is given below:
\begin{align*}
	\hat v(z)
		& = \hat v_b(z) + \hat v_c(z) \\
		& = \frac{\sum_{z' \in U(z)} \pi(z') v(z')}{\sum_{z' \in U(z)} \pi(z')} + \sum_{h \cdot a \in K(z)} \left(\frac{\sum_{a' \in A(U(h))} \sum_{h' \in U(h)} \pi(h' \cdot a') v'(h' \cdot a')}{\sum_{h' \in U(h)} \pi(h')} \right. \\
		& \qquad \left. - \frac{\sum_{h' \in U(h)} \pi(h' \cdot a) v'(h' \cdot a)}{\sum_{h' \in U(h)} \pi(h' \cdot a)}\right) \\
		& = \sum_{z' \in U(z)} \frac{\pi(z') v(z')}{\sum_{z'' \in U(z)} \pi(z'')} \\
			& \qquad + \sum_{h \cdot a \in K(z)} \sum_{a' \in A(U(h))} \sum_{h' \in U(h)} \frac{\pi(h' \cdot a')}{\sum_{h'' \in U(h)} \pi(h'')} v'(h' \cdot a') \\
			& \qquad + \sum_{h \cdot a \in K(z)} \sum_{h' \in U(h)} \left( -\frac{\pi(h' \cdot a)}{\sum_{h'' \in U(h)} \pi(h'' \cdot a)} \right) v'(h' \cdot a).
\end{align*}
For the sake of simplicity, define the affine shift $b(z)$,
\[
	b(z) = \sum_{z' \in U(z)} \frac{\pi(z') v(z')}{\sum_{z'' \in U(z)} \pi(z'')},
\]
the sets $S(z)$ and $S'(z)$ of $H$ (note that $S'(z) \subseteq S(z)$),
\[
	S(z) = \{h' \cdot a': \exists h \cdot a \in K(z), a' \in A(U(h)), h' \in U(h)\}, \quad S'(z) = \{h' \cdot a: \exists h \cdot a \in K(z), h' \in U(h)\},
\]
and $\mathbf{c}(z) \in \R^H$, the vector of coefficients of each $v'(\cdot)$.
Explicitly, for each $h \cdot a \in S(z) \setminus S'(z)$,
\[
	\mathbf{c}(z)_{h \cdot a} = \frac{\pi(h \cdot a)}{\sum_{h' \in U(h)} \pi(h')},
\]
for each $h \cdot a \in S'(z)$,
\begin{equation}
	\mathbf{c}(z)_{h \cdot a} = \frac{\pi(h \cdot a)}{\sum_{h' \in U(h)} \pi(h')} - \frac{\pi(h \cdot a)}{\sum_{h' \in U(h)} \pi(h' \cdot a)},
	\label{eqn:negative-coefficient}
\end{equation}
and finally for the remaining $h \notin S(z)$,
\[
	\mathbf{c}(z)_{h} = 0.
\]
Putting it all together, we obtain the following simplified expression:
\begin{equation}
	\hat v(z) = b(z) + \sum_{h \in H} \mathbf{c}(z)_h v'(h).
\end{equation}

\section{Detailed refinement of the optimization problem for sample variance}
\label{sec:detailed}

The detailed derivation process for the refinement of the proxy objective of the optimization problem given by~\citet{whiteandbowling2009} in (\ref{eqn:optimization-problem}) is as follows:
\begin{align*}
	\underset{\bm{\theta} \in \R^H}{\textbf{Minimize: }} C(\bm{\theta})
		& = \sum_{t=1}^T\left(\hat v_{\bm{\theta}}(z_t) - \frac{1}{T} \sum_{t'=1}^T \hat v_{\bm{\theta}}(z_{t'})\right)^2 \\
		& = \sum_{t=1}^T\left(b(z_t) + \langle \mathbf{c}(z_t), \bm{\theta} \rangle - \frac{1}{T} \sum_{t'=1}^T \left( b(z_{t'}) + \langle \mathbf{c}(z_{t'}), \bm{\theta} \rangle \right) \right)^2 \\
		& = \sum_{t=1}^T\left(b(z_t) + \langle \mathbf{c}(z_t), \bm{\theta} \rangle - \left( \bar{b} + \langle \bar{\mathbf{c}}, \bm{\theta} \rangle \right) \right)^2 \\
		& = \sum_{t=1}^T\left(\left(b(z_t) - \bar{b}\right) + \left< \mathbf{c}(z_t) - \bar{\mathbf{c}}, \bm{\theta} \right> \right)^2,
\end{align*}
where $\bar{b} = \frac{1}{T} \sum_{t=1}^T b(z_t)$ and $\bar{\mathbf{c}} = \frac{1}{T} \sum_{t=1}^T \mathbf{c}(z_{t})$.

\section{Omitted proofs}
\label{sec:proofs}

This section contains several proofs that were omitted in the main body of the paper.

\subsection{Proof of Proposition~\ref{pro:optimal}}
\label{sec:optimal}

\begin{proof}
	The cost function $C(\bm{\theta}) = \sum_{t=1}^T\left(\left(b(z_t) - \bar{b}\right) + \left< \mathbf{c}(z_t) - \bar{\mathbf{c}}, \bm{\theta} \right> \right)^2$ can be rewritten in the standard least-squares form as follows: $C(\bm{\theta}) = \|\bm{y} + \bm{X}\bm{\theta}\|_2^2$, where $\bm{X} = \begin{bmatrix} (\mathbf{c}(z_t) - \bar{\mathbf{c}})^\top \end{bmatrix}_{t \in \{1,\hdots,T\}} \in \mathbb{R}^{T \times H}$ and $\bm{y} = \begin{bmatrix} b(z_t) - \bar{b} \end{bmatrix}_{t \in \{t,\hdots,T\}} \in \mathbb{R}^T$.
	When expanded, we have $C(\bm{\theta}) = \bm{y}^\top \bm{y} + 2 \bm{y}^\top \bm{X}\bm{\theta} + \bm{\theta}^\top \bm{X}^\top \bm{X}\bm{\theta}$.
	Since $\bm{X}^\top \bm{X}$ is positive semidefinite, we have that $C(\bm{\theta})$ is convex.
	Therefore, it has a global minimum, and there exists a vector $\bm{\theta}^* \in \mathbb{R}^H$ that reaches it, as required.
\end{proof}

\subsection{Proof of Proposition~\ref{pro:ivw}}
\label{sec:ivw}

\begin{proof}
	Let $\hat v(z_1), \hdots, \hat v(z_T)$ be independent unbiased AIVAT estimates with variances $\Var(\hat v(z_1)), \hdots, \Var(\hat v(z_T))$.
	The weighted average $\bar v^* = \sum_{t=1}^T w_t^* \hat v(z_t)$ has variance $\Var(\bar v^*) = \sum_{t=1}^T (w_t^*)^2 \Var(\hat v(z_t))$.
	It is a classical result that this variance can be minimized by setting each weight $w_t$ to be proportional to $\frac{1}{\Var(\hat v(z_t))}$; cf.~\citet[Ch.~4]{hartungetal2011}.
\end{proof}

\subsection{Proof of Proposition~\ref{pro:estimator-bias}}
\label{sec:estimator-bias}

\begin{proof}
	We have that
	\begin{align*}
		\E[\bar v^*|\sigma]
			& = \E\left[\sum_{t=1}^T w_t^* \hat v(z_t)\middle|\sigma\right] \\
			& = \sum_{t=1}^T \E[w_t^* \hat v(z_t)|\sigma] \\
			& = \sum_{t=1}^T \left(\E[w_t^*|\sigma]\E[\hat v(z_t)|\sigma] + \Cov(w_t^*, \hat v(z_t)|\sigma)\right) \\
			& = \sum_{t=1}^T \E[w_t^*|\sigma]\E[\hat v(z_t)|\sigma] + \sum_{t=1}^T \Cov(w_t^*, \hat v(z_t)|\sigma) \\
			& = \E_{z \in Z}[v(z)|\sigma] \sum_{t=1}^T \E[w_t^*|\sigma] + \sum_{t=1}^T \Cov(w_t^*, \hat v(z_t)|\sigma) \\
			& = \E_{z \in Z}[v(z)|\sigma] \E\left[\sum_{t=1}^T w_t^*\middle|\sigma\right] + \sum_{t=1}^T \Cov(w_t^*, \hat v(z_t)|\sigma).
	\end{align*}
	We assume normalized weights, so $\E[\bar v^*|\sigma] = \E_{z \in Z}[v(z)|\sigma] + \sum_{t=1}^T \Cov(w_t^*, \hat v(z_t)|\sigma)$.
	If $\sum_{t=1}^T \Cov(w_t^*, \hat v(z_t)|\sigma) = 0$, then $\E[\bar v^*|\sigma] = \E_{z \in Z}[v(z)|\sigma]$, as required.
	And if $\sum_{t=1}^T \Cov(w_t^*, \hat v(z_t)|\sigma) \ne 0$, then $\E[\bar v^*|\sigma] \ne \E_{z \in Z}[v(z)|\sigma]$, as required.
	Let $S = \sum_{t = 1}^T w_t$ and let $w$ and $v(z)$ be random variables for $w_t$ and $v(z_t)$.
	Since we assume the data points are i.i.d., we have that $\E[w_t] = \E[w]$, $\E[v(z_t)] = \E[v(z)]$, and $\E[S] = T \E[w]$.
	Then, we have that
	\begin{align*}
		\bias
			& = \sum_{t=1}^T \Cov(w_t^*, \hat v(z_t)|\sigma) \\
			& = \sum_{t=1}^T \Cov\left(\frac{w_t}{S}, \hat v(z_t)\middle|\sigma\right) \\
			& \approx \sum_{t=1}^T \left(\frac1{\E[S]} \Cov(w_t, \hat v(z_t)|\sigma) - \frac{\E[w_t]}{\E[S]^2} \Cov(S, \hat v(z_t)|\sigma)\right) \\
			& = \frac T{\E[S]} \Cov(w, \hat v(z)|\sigma) - \frac{T \E[w]}{\E[S]^2} \Cov(S, \hat v(z)|\sigma) \\
			& = \frac T{T \E[w]} \Cov(w, \hat v(z)|\sigma) - \frac{T \E[w]}{(T \E[w])^2} \Cov(S, \hat v(z)|\sigma) \\
			& = \frac1{\E[w]} \Cov(w, \hat v(z)|\sigma) - \frac1{T \E[w]} \Cov(S, \hat v(z)|\sigma) \\
			& \approx \frac{\Cov(w, \hat v(z)|\sigma)}{\E[w]},
	\end{align*}
	as required.
\end{proof}

\section{Rules of Texas hold'em}
\label{sec:texas-holdem}

In Texas hold'em, the game begins with every player being dealt two private cards.
This is followed by a preflop betting round, where players take turns betting, matching (\textit{i.e.}, checking or calling), or giving up (\textit{i.e.}, folding) based on their hidden cards and beliefs about what cards others hold.
After the betting round, the dealer reveals three public cards as part of the flop, which is then followed by another betting round.
After that, a single card is revealed as the turn card, and another betting round ensues.
Finally, the river card is revealed publicly, and the final betting round takes place, after which the game terminates.
The game can also end when only one player remains in the pot (as others have folded).
In no-limit settings, players are allowed to bet or raise whatever amount they want (capped by their number of chips) however many times they want, whereas this amount and the maximum number of bets are determined by the rules of the game in fixed-limit settings.

\section{Feature engineering of states in 6-max no-limit Texas hold'em}
\label{sec:feature-engineering}

For the feature engineering of states in 6-max no-limit Texas hold'em, we took an approach similar to that of~\citet{whiteandbowling2009} in their evaluation of 6-player limit Texas hold'em data.
Specifically, we extracted the pot amount, and each player's hand strength and hand strength squared values.
The hand strength is the expected probability of beating another random hand, while the hand strength squared is the expected squared probability of beating a random hand.
All hand strength (squared) values were also multiplied by the pot amount and exponentiated by the number of non-folded players.

\section{Standard error of the (weighted) mean}
\label{sec:standard-error-of-the-weighted-mean}

In this section, we provide the definition of the standard error of the (weighted) mean.
When unweighted, the standard definition in the field of statistics is
\[
	SE = \frac{s}{\sqrt{N}},
\]
where $s$ is the sample standard deviation and $N$ is the sample size.
Note that
\[
	s^2 = \frac{\sum_{i=1}^N (x_i - \bar x)^2}{N - 1},
\]
where $x_1, \dots, x_N$ are the sample values and $\bar x$ is the sample mean.
Under inverse-variance weighting,~\citet{kirchner} gives the following definition:
\[
	SE^* = \frac{s^*}{\sqrt{N}},
\]
where $s^*$ is the weighted sample standard deviation and $w_1, \dots, w_N$ are the weights corresponding to the sample values $x_1, \dots, x_N$, respectively.
Here,
\[
	(s^*)^2 = \left( \frac{\sum_{i=1}^N w_i (x_i - \bar x^*)^2}{\sum_{i=1}^N w_i} \right) \left( \frac{N}{N - 1} \right),
\]
where $\bar x^* = \frac{\sum_{i=1}^N w_i x_i}{\sum_{i=1}^N w_i}$ is the weighted sample mean.

\section{Additional experimental results on heuristic uncertainty}
\label{sec:observations}

In learning the heuristic value function for our MIVAT estimator, we chose two more models that allow the calculation of output variance: \textit{Bayesian ridge (BR)}~\cite{tipping} and \textit{automatic relevance determination (ARD) regression}~\cite{mackay}.
We refer to the two new MIVAT estimators as follows: MIVAT-BR and MIVAT-ARDR, respectively.
Due to the inherent limitation of the models, we assumed for the two estimators that the heuristic value outputs are uncorrelated, \textit{i.e.}, (\ref{eqn:estimate-variance}) simplifies to (\ref{eqn:estimate-variance-uncorrelated}).
This limitation is why we relegated the use of these models to the appendix.
Except for subsampling, the training procedure for both estimators was identical to that of MIVAT-GPR.

\begin{table}
	\centering
	\caption{
		Results on the Pluribus data in milli-big blinds per hand (mbb/h).
		$SE$ stands for the standard error of the (weighted) mean, and `Est.' is an abbreviation for `Estimated'.
	}
	\label{tab:results-pluribus2}
	\begin{tabular}{cc|ccc}
		\toprule
		Estimator & Weighting & Win rate & $SE$ & Est. bias \\
		\midrule
		\multirow{2}{*}{MIVAT-BR} & Uniform & -80 & 83 & -- \\
		& IVW & 0 & \textbf{74} & 80 \\
		\multirow{2}{*}{MIVAT-ARDR} & Uniform & -80 & 83 & -- \\
		& IVW & 0 & \textbf{74} & 80 \\
		\multirow{2}{*}{MIVAT-GPR} & Uniform & -25 & 99 & -- \\
		& IVW & -22 & \textbf{75} & 3 \\
		\midrule
		Raw & Uniform & -70 & 88 & -- \\
		MIVAT-WB & Uniform & -98 & 85 & --\\
		AIVAT & Uniform & 48 & 25 & -- \\
		\bottomrule
	\end{tabular}
\end{table}

Table~\ref{tab:results-pluribus2} shows the performance of both MIVAT-BR and MIVAT-ARDR, in addition to MIVAT-GPR and our baselines.
When taking the simple mean, both MIVAT-BR and MIVAT-ARDR outperform MIVAT-GPR, but this is understandable, as we used far less data to train GPR due to computational constraints.
IVW allowed our MIVAT estimators to yield the best performance (save for AIVAT).
The difference in the uncertainty achieved between uniform averaging and IVW averaging is smaller for MIVAT-BR and MIVAT-ARDR than for MIVAT-GPR.
In general, for all our MIVAT estimators under IVW, the bias is approximately equal to the difference between the estimates under different weighting schemes.
Indeed, offsetting the IVW estimates using the estimated bias recovers an estimator that is in line with using uniform averaging, which is unbiased.
Also, the IVW win rates for MIVAT-BR and MIVAT-ARDR are not exactly zero but are shown so purely due to rounding coincidence.
Before rounding, the values were approximately $-0.233$ for MIVAT-BR and $-0.217$ for MIVAT-ARD.
Note that the values for MIVAT-BR and MIVAT-ARDR are similar because both BR and ARD are Bayesian linear regression models.

\section{Linear value function}
\label{sec:closed-form-solution}

Define a feature engineering function $\bm{\phi}: H \mapsto \R^d$ that maps a state $h$ to a respective vector of features.
When linear, we require the value function to be of the following form:
\begin{equation}
	v_{\bm{\theta}}'(h) = \bm{\phi}(h)^\top \bm{\theta},
	\label{eqn:linear-value-function}
\end{equation}
where $\bm{\theta} \in \R^d$ is the parameter vector.
Then,
\begin{align*}
	& \hat v_{\bm{\theta}}(z) = b(z) + \sum_{h \in H} \mathbf{c}(z)_h v_{\bm{\theta}}'(h) = b(z) + \sum_{h \in H} \mathbf{c}(z)_h \bm{\phi}(h)^\top \bm{\theta} \\
	& \qquad = b(z) + \left(\sum_{h \in H} \mathbf{c}(z)_h \bm{\phi}(h)\right)^\top \bm{\theta} = b(z) + \bm{\psi}(z)^\top \bm{\theta},
\end{align*}
where $\bm{\psi}(z) = \sum_{h \in H} \mathbf{c}(z)_h \bm{\phi}(h)$.

\subsection{Closed-form solution}

For MIVAT,~\citet{whiteandbowling2009} derived a closed-form formula of the optimal linear function parameter $\bm{\theta}^*$ for the optimization problem shown in (\ref{eqn:optimization-problem}).
The steps they provide can be closely followed to yield an analogous solution for AIVAT.
With $\bar b = \frac{1}{T} \sum_{t=1} b(z_t)$ and $\overline{\bm{\psi}} = \frac{1}{T} \sum_{t=1}^T \bm{\psi}(z_t)$, the parameter of the linear function that optimizes (\ref{eqn:optimization-problem}) is as follows:\footnote{
	A careful reader may notice that the latter factor in (\ref{eqn:closed-form-solution}) is negated in the equivalent formulation of~\citet{whiteandbowling2009}.
	This is because we hide the negative sign inside $\mathbf{c}(z)_{h' \cdot a}$; see (\ref{eqn:negative-coefficient}).
}\textsuperscript{,}\footnote{
	The steps we took to obtain this are isolated in Appendix~\ref{sec:closed-form-solution-2}.
}
\begin{equation}
	\bm{\theta}^* = \left(\left(\frac{1}{T} \sum_{t=1}^T \bm{\psi}(z_t) \bm{\psi}(z_t)^\top\right) - \overline{\bm{\psi}}\,\overline{\bm{\psi}}^\top \right)^{-1} \left(\bar b\,\overline{\bm{\psi}} - \left(\frac{1}{T} \sum_{t=1}^T b(z_t)\bm{\psi}(z_t)\right)\right).
	\label{eqn:closed-form-solution}
\end{equation}

We can simplify this further to
\begin{equation}
	\bm{\theta}^* = \mathbf{X}^{-1} \mathbf{y},
	\label{eqn:closed-form-solution-2}
\end{equation}
where $\mathbf{X} = \left(\left(\frac{1}{T} \sum_{t=1}^T \bm{\psi}(z_t) \bm{\psi}(z_t)^\top\right) - \overline{\bm{\psi}}\,\overline{\bm{\psi}}^\top\right)$ and $\mathbf{y} = \left(\bar b\,\overline{\bm{\psi}} - \left(\frac{1}{T} \sum_{t=1}^T b(z_t)\bm{\psi}(z_t)\right)\right)$.
Although left unstated by~\citet{whiteandbowling2009}, for the above closed-form solution to exist, $\mathbf{X}$ must be invertible.
We provide the conditions required for this to be true.

\begin{theorem}
	If no hyperplane in $\R^d$ contains all of $\{\bm{\psi}(z_t)\}_{t=1}^T$, then $\mathbf{X}$ is invertible.
\end{theorem}
\begin{proof}
	We prove the contrapositive: if $\mathbf{X}$ is non-invertible, then there is a hyperplane in $\R^d$ containing all of $\{\bm{\psi}(z_t)\}_{t=1}^T$.

	If $\mathbf{X}$ is singular, then $\exists \mathbf{a} \in \R^d$ such that $\mathbf{a} \ne \mathbf{0}$ and $\mathbf{X}\mathbf{a} = \mathbf{0}$. The latter implies $\mathbf{a}^\top\mathbf{X}\mathbf{a} = 0$. Using
	\[
		\mathbf{X} = \left(\left(\frac{1}{T} \sum_{t=1}^T \bm{\psi}(z_t) \bm{\psi}(z_t)^\top\right) - \overline{\bm{\psi}}\,\overline{\bm{\psi}}^\top\right) = \frac{1}{T} \left(\sum_{t=1}^T (\bm{\psi}(z_t) - \overline{\bm{\psi}})(\bm{\psi}(z_t) - \overline{\bm{\psi}})^\top\right),
	\]
	we obtain
	\[
		0 = \mathbf{a}^\top\mathbf{X}\mathbf{a} = \mathbf{a}^\top \frac{1}{T} \left(\sum_{t=1}^T (\bm{\psi}(z_t) - \overline{\bm{\psi}})(\bm{\psi}(z_t) - \overline{\bm{\psi}})^\top\right) \mathbf{a} = \frac{1}{T} \sum_{t=1}^T ((\bm{\psi}(z_t) - \overline{\bm{\psi}})^\top \mathbf{a})^2,
	\]
	implying, $\forall t \in \{1, \hdots, T\}: (\bm{\psi}(z_t) - \overline{\bm{\psi}})^\top \mathbf{a} = 0$, and hence $\bm{\psi}(z_t)^\top \mathbf{a} = \overline{\bm{\psi}}^\top \mathbf{a}$. Note that $\overline{\bm{\psi}}^\top \mathbf{a} = c$ is a constant and, defining $\mathbf{x} \in \R^d$, $\mathbf{x}^\top \mathbf{a} = c$ is an equation of a hyperplane.
\end{proof}
This prohibits one from applying common feature engineering techniques such as using separate linear coefficients for different stages in the game (\textit{e.g.}, in poker: preflop, flop, turn, and river).
On top of this, it lacks any regularization, which makes it prone to overfitting.
It is also unclear how one can estimate the variance of its outputs.

\section{Derivation of the closed-form solution}
\label{sec:closed-form-solution-2}

This appendix section contains the steps to obtain the optimal linear value function parameter for the optimization problem given by~\citet{whiteandbowling2009} optimization problem stated in (\ref{eqn:optimization-problem}).
Note that this is for the case of the linear heuristic value function shown in (\ref{eqn:linear-value-function}), not the parameterized version in (\ref{eqn:parameterized-value-function}).
Restating the optimization problem,
\begin{align*}
	\underset{\bm{\theta} \in \R^d}{\textbf{Minimize: }} C(\bm{\theta})
		& = \sum_{t=1}^T \left(\hat v_{\bm{\theta}}(z_t) - \frac{1}{T} \sum_{t'=1}^T \hat v_{\bm{\theta}}(z_{t'})\right)^2 \\
		& = \sum_{t=1}^T \left(\left(b(z_t) + \bm{\psi}(z_t)^\top \bm{\theta}\right) - \frac{1}{T} \sum_{t'=1}^T \left(b(z_{t'}) + \bm{\psi}(z_{t'})^\top \bm{\theta}\right)\right)^2 \\
		& = \sum_{t=1}^T \left(\left(b(z_t) - \frac{1}{T} \sum_{t'=1}^T b(z_{t'})\right) + \left(\bm{\psi}(z_t) - \frac{1}{T} \sum_{t'=1}^T \bm{\psi}(z_{t'})\right)^\top \bm{\theta}\right)^2 \\
		& = \sum_{t=1}^T \left((b(z_t) - \bar b) + \left(\bm{\psi}(z_t) - \overline{\bm{\psi}}\right)^\top \bm{\theta}\right)^2.
\end{align*}
They then took the derivative of $C$ with respect to $\bm{\theta}$ and set it to zero:
\begin{equation*}
	0 = \left. \frac{\partial C}{\partial\bm{\theta}}\right|_{\bm{\theta}=\bm{\theta}^*} = \sum_{t=1}^T 2 \left(\bm{\psi}(z_t) - \overline{\bm{\psi}}\right) \left((b(z_t) - \bar b) + \left(\bm{\psi}(z_t) - \overline{\bm{\psi}}\right)^\top \bm{\theta}^*\right).
\end{equation*}
Simplifying and solving for $\bm{\theta}^*$, the following closed-form solution was obtained:
\[
	\bm{\theta}^* = \left(\left(\frac{1}{T} \sum_{t=1}^T \bm{\psi}(z_t) \bm{\psi}(z_t)^\top\right) - \overline{\bm{\psi}}\,\overline{\bm{\psi}}^\top \right)^{-1} \left(\bar b\,\overline{\bm{\psi}} - \left(\frac{1}{T} \sum_{t=1}^T b(z_t)\bm{\psi}(z_t)\right)\right).
\]

\section{Testbench specification}
\label{sec:resources}

Our testbench contains an AMD Ryzen 9 3900X 12-core, 24-thread desktop processor and 128 GB memory.

%%%%%%%%%%%%%%%%%%%%%%%%%%%%%%%%%%%%%%%%%%%%%%%%%%%%%%%%%%%%

% \newpage
% \input{checklist.tex}

\end{document}